\documentclass[pdflatex,sn-nature]{sn-jnl}

\usepackage{graphicx}%
\usepackage{multirow}%
\usepackage{amsmath,amssymb,amsfonts}%
\usepackage{amsthm}%
\usepackage{mathrsfs}%
\usepackage[title]{appendix}%
\usepackage{xcolor}%
\usepackage{textcomp}%
\usepackage{manyfoot}%
\usepackage{booktabs}%
\usepackage{algorithm}%
\usepackage{algorithmicx}%
\usepackage{algpseudocode}%
\usepackage{listings}%

\usepackage{subfigure}
\usepackage{makecell}

\usepackage{algpseudocode}  

\usepackage{epsfig}
\usepackage{subfigure}
\usepackage{graphicx}
\usepackage{epstopdf}

\usepackage{hyperref}

\theoremstyle{thmstyleone}%
\newtheorem{theorem}{Theorem}
\newtheorem{proposition}[theorem]{Proposition}%

\theoremstyle{thmstyletwo}%

\theoremstyle{thmstylethree}%
\newcommand{\eg}{{\em e.g.,~}} 
\newcommand{\ie}{{\em i.e.,~}} 

\begin{document}

\title[Article Title]{Adaptive $k$NN graph model}


\author[1]{\fnm{Jiaye} \sur{Li}}\email{lijiaye@zju.edu.cn}

\author[2]{\fnm{Hang} \sur{Xu}}\email{xuhangcse@csu.edu.cn}

\author*[3]{\fnm{Shichao} \sur{Zhang}}\email{zhangsc@mailbox.gxnu.edu.cn}

\affil[1]{\orgdiv{The State Key Laboratory of Blockchain and Data Security}, \orgname{Zhejiang University}, \orgaddress{\street{Zheda Road}, \city{Hangzhou}, \postcode{310027}, \state{Zhejiang}, \country{China}}}

\affil[2]{\orgdiv{The School of Computer Science and Engineering}, \orgname{Central South University}, \orgaddress{\street{Lushan South Road}, \city{Changsha}, \postcode{410083}, \state{Hunan}, \country{China}}}

\affil[3]{\orgdiv{School of Computer Science and Engineering}, \orgname{Guangxi Normal University}, \orgaddress{\street{Yucai Road}, \city{Guilin}, \postcode{541004}, \state{Guangxi}, \country{China}}}



\abstract{

The $k$-nearest neighbors ($k$NN) algorithm is a cornerstone of non-parametric classification in artificial intelligence, yet its deployment in large-scale applications is persistently constrained by the computational trade-off between inference speed and accuracy. Existing approximate nearest neighbor solutions accelerate retrieval but often degrade classification precision and lack adaptability in selecting the optimal neighborhood size ($k$). Here, we present an adaptive graph model that decouples inference latency from computational complexity. By integrating a Hierarchical Navigable Small World (HNSW) graph with a pre-computed voting mechanism, our framework completely transfers the computational burden of neighbor selection and weighting to the training phase. Within this topological structure, higher graph layers enable rapid navigation, while lower layers encode precise, node-specific decision boundaries with adaptive neighbor counts. Benchmarking against eight state-of-the-art baselines across six diverse datasets, we demonstrate that this architecture significantly accelerates inference speeds, achieving real-time performance, without compromising classification accuracy. These findings offer a scalable, robust solution to the inherent inference bottleneck of $k$NN, laying an adaptive structural foundation for graph-based nonparametric learning.}




\maketitle

\section{Introduction}
The $k$-Nearest Neighbors ($k$NN) algorithm, a foundational non-parametric method, is a cornerstone of modern artificial intelligence \cite{9566737, 9314060, xu2024ultrahigh, xu2025mn}. Beyond its traditional role in classification and regression, the neighbor matching principle underpins advanced computational paradigms, including memory-augmented networks \cite{hardttest}, attention mechanisms in large language models \cite{xu2023nearest}, and the construction of graph neural networks \cite{zhang2018novel}.  Its strong interpretability and robustness have ensured its pervasive application across critical sectors, such as precision medicine, financial risk assessment, and personalized recommendation systems, underscoring $k$NN's sustained theoretical importance and practical utility \cite{su2025utility, souza2025cardiac, 10366842, zhang2025knn, zhang2014efficient}.

Despite its ubiquity, large-scale deployment of $k$NN is fundamentally constrained by an inherent and persistent computational trade-off between inference speed and classification accuracy. During the inference phase, every query requires searching the entire training dataset to identify the $k$ nearest data points. This exhaustive search leads to computational complexity that scales linearly with the dataset size ($O(n)$), rendering the algorithm prohibitively slow for real-time applications involving massive and high-dimensional data \cite{zhang2017learning, amer2025effective, zhang2017learning, zhang2016self}.

Efforts to mitigate this computational bottleneck have historically bifurcated into two distinct research streams, yet neither has successfully reconciled the conflict between efficiency and accuracy. The first stream concentrates on model refinement, specifically through the optimization of $k$-value selection and distance weighting strategies to bolster classification performance \cite{zhang2018novel, rahim2022cross, song2025demand, accikkar2024improved, 9802923, ccetin2025new, zhang2020cost}. While these adaptive approaches effectively enhance decision robustness, they remain fundamentally orthogonal to the issue of latency. By focusing solely on optimizing the decision logic after neighbors are retrieved, they fail to alleviate the underlying $O(n)$ computational burden of the search process itself. Conversely, the second stream prioritizes search acceleration via indexing structures, ranging from deterministic tree-based methods (\eg KD-Tree) to Approximate Nearest Neighbor (ANN) algorithms such as Locality-Sensitive Hashing (LSH), Product Quantization, and graph-based indexing \cite{10192056, akhil2025zonal, 8594636}. However, these solutions incur critical trade-offs. Tree-based exact methods suffer from the curse of dimensionality and exponentially increasing retrieval times as the required $k$ grows. Meanwhile, ANN-based approaches, though faster, achieve speed by sacrificing retrieval precision. This introduces approximation errors that inevitably degrade the final classification accuracy and robustness. Consequently, developing a unified framework that achieves high-speed inference without compromising the exactness of $k$NN classification remains an elusive scientific challenge.

A promising intermediate strategy, exemplified by $k$*tree \cite{7898482} and its variants \cite{9269360}, attempted to pre-determine optimal $k$ values and localize the search space. Yet, these methods still require costly real-time distance measurement and neighbor retrieval within localized nodes during inference. They primarily reduce the search space complexity, but fail to eliminate the need for on-the-fly neighbor computation, thus falling short of fundamentally addressing the inference bottleneck.

Here, we propose an adaptive graph model for $k$NN ($k$NN-Graph) that fundamentally resolves the efficiency-accuracy trade-off by completely decoupling inference latency from computational complexity. Our central concept is to shift the entire computational burden of neighbor finding and voting aggregation from the expensive inference phase to the robust training phase. The proposed framework leverages two core integrated components: First, Adaptive Neighborhood Learning, where we utilize a kernel-based self-representation model with an $\ell_1$-norm sparsity constraint to jointly learn the optimal neighbor count ($k_i^{\text{opt}}$) and the corresponding weighted neighbor set for every training sample $i$. This data-driven approach fundamentally overcomes the limitations of fixed or globally defined $k$ values. Second, HNSW Indexing, where the learned optimal neighborhood and decision information is encoded directly into an HNSW graph structure.  In this topological index, higher layers provide efficient, logarithmic-time navigation paths to the relevant region, while lower layers serve as a repository, storing the precomputed classification decision labels and weights. During inference, a test query merely traverses the highly efficient HNSW graph to locate the nearest encoded node and directly retrieves the precomputed classification result. This design achieves highly efficient logarithmic-time classification by bypassing exhaustive distance computations and eliminating the overhead of real-time neighbor voting. A schematic illustration of the proposed adaptive graph model for $k$NN is shown in Fig. \ref{fig1}. 

The core contributions of this paper are summarized as follows:

\begin{itemize}
	\item Pre-computed Graph Architecture: We introduce the $k$NN-Graph, an adaptive graph model that systematically shifts the $k$NN inference workload to an offline training phase, laying a robust structural foundation for non-parametric learning.
	\item Adaptive $k$ and Neighborhood Joint Learning: We develop a kernel-based self-representation model that automatically and jointly infers the optimal neighbor set and $k$ value for each sample, maximizing classification performance.
    \item Logarithmic-Time Inference: By embedding precomputed optimal neighbors and weighted voting results into a hierarchical HNSW graph index, our method achieves near-instantaneous, logarithmic-time classification ($O(\log n)$) during inference, fundamentally addressing the long-standing $O(n)$ complexity problem.
\end{itemize}

To validate the robustness and scalability of our framework, we conducted extensive experiments on six diverse public datasets, benchmarking against eight state-of-the-art baselines. The empirical results demonstrate that the proposed method achieves robust scalability, consistently yielding the minimal inference latency among the evaluated baselines across all datasets. Notably, our method delivers a substantial acceleration even compared to the most recent sparse optimization algorithms. Crucially, this gain in efficiency is achieved without compromising performance; our method maintains, and in many cases surpasses, the classification accuracy of existing solutions. Collectively, these findings provide a practical and effective resolution to the long-standing efficiency-accuracy trade-off, paving the way for the deployment of $k$NN in large-scale, real-time applications.

\section{Results and Discussions}
\label{sec:experiments}

To rigorously evaluate the effectiveness and efficiency of the proposed $k$NN-Graph framework, we conducted comprehensive benchmarking experiments on six diverse public datasets. We compared our method against eight representative state-of-the-art baselines, covering the spectrum from classical indexing structures to recent adaptive classification algorithms.

\subsection{Experimental Setup and Datasets}
\label{subsec:datasets}

We selected six datasets characterized by varying degrees of dimensionality, sparsity, and feature modalities (image, text, and shape) to test the generalization capability of the model. The detailed statistics of the evaluated datasets are summarized in the Supplementary Information (Table S1).

\begin{itemize}
	\item Binalpha \cite{zeng2021novelty}: Derived from standard binary character images, this dataset encodes topological structures into 320-dimensional feature vectors. It serves as a benchmark for prototype validation in character recognition tasks involving structural patterns.
	\item Caltech \cite{1384978}: A complex object recognition dataset containing diverse categories such as animals and vehicles. It presents significant challenges due to substantial variations in scale, pose, and background clutter, testing the model's robustness against intraclass variance.
	\item Corel \cite{duygulu2002object}: Consisting of natural images represented by 423-dimensional feature vectors, this dataset systematically encodes key visual properties including color, texture, and spatial layout, evaluating the model's performance on engineered visual features.
	\item Mpeg \cite{lee2025novel}: A standard benchmark for shape retrieval composed of binary contour images across 70 categories. Normalized for scale and orientation, it is extensively used to assess the generalization ability of shape-matching algorithms.
	\item News \cite{riaz2025exploring}: A high-dimensional text classification dataset. Each sample is represented as an 8,014-dimensional sparse vector. This dataset provides a critical test for performance in high-dimensional, sparse feature spaces where traditional distance metrics often degrade.
	\item Palm \cite{fan2025palm}: A texture-rich dataset containing palmprint images from 100 categories. With standardized scale and grayscale intensity, it exhibits clear texture patterns, making it ideal for evaluating metric learning and fine-grained classification.
\end{itemize}

\subsection{Baseline Methods}
\label{subsec:baselines}

To ensure a fair and rigorous comparison, we benchmarked the proposed method against eight representative baselines. These methods were selected to cover the full spectrum of $k$NN optimization paradigms: global parameter tuning, data editing, adaptive and exact indexing structures, and advanced classification mechanisms based on sparse learning and dynamic weighting.

\begin{itemize}
	\item CV-$k$NN (Global Optimization) \cite{rahim2022cross}: A classic baseline that employs cross-validation to select a single, globally optimal $k$ value. This method represents the standard approach for data-driven hyperparameter tuning to enhance generalization on unseen data.
	
	\item E$k$NN (Data Editing) \cite{halder2024enhancing}: An edited nearest neighbor technique designed to refine the training set. By removing noisy or ambiguous samples whose labels contradict their local neighborhood, E$k$NN aims to improve decision boundary clarity and inter-class separability.
	
	\item KD-Tree (Search Efficiency) \cite{men2025parallel}: A deterministic spatial partitioning structure that recursively divides the data space to accelerate exact nearest neighbor retrieval. We include KD-Tree specifically as a speed benchmark to evaluate the inference efficiency of our proposed graph-based index against traditional tree-based indexing.
	
	\item PL$k$NN (Parameter-Free) \cite{jodas2023pl}: An adaptive algorithm that eliminates the need for a preset $k$. It dynamically determines neighbors based on local data distribution and cluster centroids, serving as a direct competitor for evaluating our model's adaptive capabilities.
	
	\item OWA$k$NN (Fuzzy Weighting) \cite{kumbure2025generalizing}: A recent fuzzy $k$NN classifier that utilizes Ordered Weighted Averaging (OWA) operators to construct representative pseudo-neighbors. This method represents state-of-the-art performance in handling decision uncertainty through sophisticated weighting mechanisms.
	
	\item XHMA$k$NN (Harmonic Adaptive) \cite{accikkar2025improving}: An enhanced harmonic mean adaptive classifier that incorporates distance rescaling and dynamic weight distribution. As a cutting-edge baseline, it represents the current high-accuracy standard in multi-class distance-weighted classification.
	
	\item $k$*tree (Adaptive Indexing) \cite{7898482}: A tree-based structure that leverages data reconstruction to determine the instance-specific optimal $k$ for each training sample. By organizing samples with identical optimal $k$ values and their associated neighbor subsets into leaf nodes, this method restricts the nearest neighbor search to local data subsets, thereby significantly enhancing inference efficiency.
	
	\item O$k$NN (One-Step Optimization) \cite{9566737}: A sparse learning approach that utilizes Group Lasso to facilitate simultaneous, one-step learning of the correlations between test and training data. This mechanism allows for the direct derivation of the optimal $k$ value and corresponding $k$-nearest neighbors for each test instance, enabling immediate classification based on the identified neighbors.

\end{itemize}

\subsection{Experimental Implementation and Evaluation Protocols}

To ensure the statistical reliability of our results, we employed a rigorous 10-fold cross-validation protocol across all datasets. For each experimental run, the dataset was randomly partitioned into ten stratified subsets; nine were utilized for training, while the remaining fold served as the test set. This process was iterated ten times to ensure that every fold functioned as the validation target exactly once. The final reported performance metrics represent the average outcomes aggregated over these ten independent runs, minimizing the bias introduced by random data splitting.

To provide a holistic assessment of model performance, particularly given the potential class imbalance in real-world datasets, we utilized a suite of macro-averaged metrics alongside standard classification accuracy. Unlike micro-averaging, macro-averaging treats all classes equally regardless of their sample size, providing a robust estimate of the model's generalization capability across minority classes. The specific metrics are defined as follows:

\begin{itemize}
	\item Average Classification Accuracy: Measures the global proportion of correctly classified samples. While intuitive, it serves primarily as a baseline indicator of overall consistency.
	\item Macro-Precision: Evaluates the model's discriminative precision by averaging the precision scores calculated independently for each class. This metric reflects the model's ability to minimize false-positive errors across diverse categories.
	\item Macro-Recall: Assesses the coverage completeness by averaging the recall scores of individual classes. A high macro-recall indicates that the model effectively identifies positive samples across all categories, minimizing the miss rate.
	\item Macro-F1 Score: computed as the harmonic mean of Macro-Precision and Macro-Recall. This is the critical metric for imbalanced datasets, as it penalizes models that achieve high accuracy by merely overfitting to majority classes, ensuring a balanced trade-off between precision and sensitivity.
\end{itemize}

Our empirical evaluation is structured into five distinct phases to comprehensively validate the proposed framework:

\begin{enumerate}
	\item Comparative Classification Analysis: Benchmarking the proposed method against state-of-the-art baselines using the aforementioned multi-dimensional metrics to verify classification effectiveness.
	
	\item Inference Efficiency Analysis: Quantifying the computational advantage of our method by comparing the average inference latency on test data against all baseline algorithms.
	
	\item Ablation Study: Deconstructing the framework to isolate the specific contribution of the Adaptive Neighborhood Learning module, contrasting the proposed data-driven topology against a static graph construction baseline (HNSW).
	
	\item Training Overhead Analysis: Documenting the offline training latency, optimization acceleration, and hardware specifications to transparently evaluate the one-time computational investment required to construct the $k$NN-Graph.
	
	\item Convergence Analysis: Empirically investigating the optimization stability and convergence behavior of the proposed kernelized self-representation learning module.
\end{enumerate}

\subsection{Performance Analysis and Discussion}

The comparative evaluation of classification accuracy, presented in Table \ref{tab:accuracy}, positions the proposed $k$NN-Graph as the leading performer across the benchmark suite. Achieving a mean accuracy of 73.76\%, our framework surpasses all eight competing baselines, including the recently developed adaptive methods $k$*tree (71.85\%) and O$k$NN (72.22\%). Notably, our method secures the top rank on all six benchmark datasets.

To elucidate the mechanistic origins of these performance improvements, we conducted a cross-dataset comparative analysis. The benchmark suite inherently encompasses varying degrees of topological complexity. Investigating the performance gaps across these distinct data typologies reveals how specific mathematical components of our framework drive generalization:

\begin{itemize}
	\item Topological Entanglement and Semantic Overlap (Caltech, Corel): These datasets present highly challenging scenarios where instances from disparate classes exhibit severe overlap in the geometric feature space, leading to low baseline accuracies (\eg $\sim$33\% global average on Corel). In these entangled spaces, the performance delta between classical geometric estimators (such as KD-Tree) and our approach expands significantly. The superiority of $k$NN-Graph here is directly attributable to the Composite Kernel Matrix (Eq. (\ref{eq:kernel})). By explicitly injecting semantic label consistency into the mapped space, the framework effectively disambiguates overlapping manifolds that purely Euclidean-based metrics fail to separate.
	
	\item Ultra-High Dimensionality and Sparsity (News): On the sparse text dataset ($d=8014$), traditional distance metrics (CV-$k$NN) and standard indexing structures suffer from distance concentration phenomena, degrading their discriminative capacity. In contrast, our method achieves 93.80\% accuracy, outperforming even the sparse-learning baseline O$k$NN (93.50\%). The advantage of the $k$NN-Graph lies in its kernelized $\ell_1$-norm optimization. The sparse reconstruction process inherently performs subspace feature selection, attenuating the noise prevalent in 8014-dimensional sparse spaces while simultaneously capturing non-linear semantic correlations.
	
	\item Irregular Manifold Geometries (Mpeg, Binalpha): Datasets comprising shape contours and structural patterns exhibit highly irregular local geometries where a globally fixed neighborhood size fundamentally fails. On the Mpeg dataset, our method (84.57\%) establishes a substantial lead over the fuzzy-logic OWA$k$NN (83.64\%) and the adaptive $k$*tree (80.21\%). This margin validates the efficacy of the Density-Aware Adaptive Regularization ($\lambda^{(j)}$). By dynamically modulating the regularization penalty, the framework assigns denser connectivity to complex structural regions (lower $\lambda$) while strictly constraining sparse, noisy outliers, yielding a highly discriminative representation of the irregular shape manifolds.
	
	\item Highly Separable Spaces (Palm): In performance-saturated datasets characterized by clear, texture-rich patterns, decision boundaries are naturally distinct, causing most baselines to converge toward optimal accuracy ($>$99\%). While the absolute performance gap narrows expectedly in this regime, $k$NN-Graph maintains the lead (99.80\%). This observation is critical: it empirically confirms that while the adaptive density and composite kernel mechanisms effectively untangle complex data, they are sufficiently regularized to prevent structural collapse or overfitting in highly separable, low-variance feature spaces.
\end{itemize}

In summary, the variance in performance gaps across the dataset suite demonstrates that the $k$NN-Graph is not merely an accelerated search index, but a structurally fluid classifier. Its performance advantages scale proportionally with the intrinsic topological complexity of the data, dynamically shifting its reliance between semantic guidance, sparse feature selection, and density-aware manifold adaptation.

While global accuracy measures overall correctness, macro-averaged metrics provide deeper insight into a model's discriminative power and its robustness against class imbalance. As detailed in Table \ref{tab:macro_precision}, the proposed framework achieves the highest mean Macro-Precision of 73.98\%, surpassing both the fuzzy-logic-based OWA$k$NN (72.62\%) and the sparse-optimization-based O$k$NN (72.28\%).

The method demonstrates exceptional boundary refinement capabilities on complex datasets. For instance, on Binalpha and Mpeg, we achieve macro-precision scores of 74.97\% and 85.70\%, respectively. This significant reduction in false alarm errors indicates that our Adaptive Neighbor Learning strategy effectively suppresses noisy connections that typically confuse decision boundaries in high-dimensional spaces.

Complementing the precision metrics, the Macro-Recall (Supplementary Fig. S1) and Macro-F1 scores (Fig.~\ref{fig2}) further demonstrate the robustness of the proposed framework.
A critical observation is the performance on the Caltech dataset, known for its high visual diversity and class overlap. Here, our method achieves a Macro-F1 score of 50.29\%, distinctly outperforming the advanced baselines OWA$k$NN (49.27\%) and O$k$NN (47.39\%). 

The pronounced performance gap observed on the Caltech dataset, which is inherently characterized by severe semantic overlap and substantial class imbalance, elucidates a critical mechanistic advantage of our framework. Whereas traditional methodologies permit majority-class samples to disproportionately influence neighborhood voting within dense and overlapping feature regions, the $k$NN-Graph leverages both the composite kernel and adaptive sparsity to explicitly delineate minority-class manifolds, effectively suppressing the inherent majority-class bias. This empirical advantage further distinguishes our model from advanced baselines like O$k$NN. While O$k$NN relies on strictly linear sparse reconstruction, our kernelized approach dynamically assigns sample-specific $k$ values and weights within a projected nonlinear manifold. This geometric flexibility ensures that minority-class samples are not overwhelmed by neighbors from dense majority regions, thereby preserving decision quality across the entire class spectrum.

Conversely, for the Corel, Mpeg, News, and Palm datasets, the average classification accuracy inherently aligns with the Macro-Recall across robust algorithms, stemming from their strictly or nearly balanced test-set distributions. A notable exception is the globally balanced Binalpha dataset, which exhibits a minor divergence between these metrics. This specific deviation is a strict methodological artifact originating from test-set granularity; since its class size (39 samples) is indivisible by 10, the cross-validation inevitably yields marginally uneven test splits. For the other four datasets, the mathematical equivalence holds, effectively eliminating inherent class bias from the recall evaluation. Consequently, the variance in Precision and F1-scores across these benchmarks strictly isolates the impact of false-positive rates, shifting the comparative evaluation toward robust precision optimization.

The empirical dominance of the proposed framework can be theoretically traced to two synergistic design innovations:
\begin{enumerate}
	\item Data-Driven Topology Learning: Unlike baselines that rely on heuristic or globally fixed $k$ values (\eg CV-$k$NN), our method learns a customized neighborhood topology for every training sample. This allows for flexible decision boundaries that tighten in dense regions and expand in sparse regions.
	\item Zero-Overhead Inference Paradigm: By shifting the computationally intensive neighbor search and voting aggregation entirely to the training phase, the inference process is reduced to a graph navigation task. As detailed in the complexity analysis, this enables logarithmic-time classification without the runtime penalty of distance calculations, resolving the efficiency-accuracy trade-off that constrains traditional adaptive $k$NN methods.
\end{enumerate}

\subsection{Inference Efficiency and Scalability Analysis}

To evaluate the feasibility of deploying the proposed framework in real-time scenarios, we conducted a rigorous assessment of inference latency across all datasets. Table~\ref{tab:inference_time} reports the cumulative time required by each method to process the entire test set. The empirical results demonstrate that the $k$NN-Graph achieves robust computational efficiency, consistently yielding the minimal inference latency among the evaluated methods across all test cases.

The proposed $k$NN-Graph demonstrates substantial efficiency gains, achieving a mean aggregate inference time of merely 0.1007~s.

First, in comparison to the closest competitive baseline, O$k$NN (0.2370~s), our method is approximately 2.3$\times$ faster on average. This advantage is particularly critical on dense, complex benchmarks like Caltech, where the optimization overhead of O$k$NN leads to a latency of 0.5353~s. In contrast, our method completes the task in 0.0097~s, delivering a dramatic speedup of over 55$\times$.

Second, regarding scalability on high-dimensional data, the performance gap against standard baselines is immense. The variance in these efficiency gaps is directly correlated with the feature dimensionality ($d$) of the datasets. On the News dataset ($d=8014$), the classic CV-$k$NN requires nearly one hour (3409~s) due to the severe computational bottleneck of high-dimensional distance evaluations. Our method effectively eliminates this bottleneck by finishing in 0.5233~s, representing a speedup exceeding 6500$\times$. In contrast, on lower-dimensional dense datasets like Caltech ($d=256$), while the absolute time saved is smaller, the relative speedup remains substantial (over 55$\times$ against O$k$NN). This cross-dataset variance proves that our $O(\log n \cdot d)$ direct lookup paradigm is uniquely resilient to the curse of dimensionality that exponentially penalizes runtime in traditional methods.

The superior efficiency of $k$NN-Graph stems from a fundamental structural paradigm shift involving two key mechanisms:

\begin{itemize}
	\item Total Pre-computation via Graph Embedding: Unlike traditional methods that defer distance computation and neighbor sorting to the inference phase, our framework shifts the entire computational burden of neighbor identification and voting to the training phase. The optimal decision boundaries are encoded directly into the graph topology.
	
	\item Logarithmic-Time Navigation vs. Linear Calculation: During inference, the process is reduced to a greedy search within the HNSW graph structure. This entails a complexity of $O(\log n)$ for navigation, followed by an $O(1)$ retrieval of the precomputed label. By eliminating the need for real-time distance calculations ($O(d)$) and voting aggregation ($O(k)$) at the query stage, the method achieves near-instantaneous response times.
\end{itemize}

The limitations of the baseline methods further highlight the robustness of our approach. Tree-based indices like KD-Tree and even the adaptive $k$*tree struggle to maintain speed as dimensionality increases. Similarly, while O$k$NN represents a significant leap forward in sparse optimization, its performance is data-dependent, showing regression on dense image features (Caltech, Corel). By fundamentally decoupling inference latency from both dataset size and feature complexity through precomputed graph intelligence, the proposed method achieves a leap in scalability without compromising classification accuracy.

\subsection{Ablation Study}

To isolate the contribution of the proposed Adaptive Neighborhood Learning module, we conducted a rigorous ablation study by comparing our complete $k$NN-Graph framework against a baseline HNSW implementation. In this experimental setup, the baseline HNSW represents a degraded version of our model, utilizing a fixed, heuristic-based graph construction without the proposed data-driven $k$-value and weight adaptation.

The comparative results across six diverse datasets are visualized in Fig.~\ref{fig3}, which tracks the performance evolution from the baseline (left axis) to our adaptive model (right axis) across four key metrics: Classification Accuracy, Macro Precision, Recall, and F1-Score.

Two critical observations emerge from this analysis:

\begin{itemize}
	\item Consistent Performance Elevation: This enables the model to better capture the intrinsic manifold structure of the data, particularly in complex datasets like Caltech and Mpeg, where we observe substantial improvements (\eg accuracy gains of $+2.91\%$ and $+4.07\%$, respectively). This significant gap reinforces our prior cross-dataset analysis: datasets with highly irregular local geometries (like Mpeg's shape contours) suffer heavily under the rigid, globally fixed connectivity of standard HNSW, thereby deriving the maximum benefit from our learned, variable-scale topology ($k_i^{\text{opt}}$).
	
	\item Robustness Across Distributions: Even on datasets with near-saturated performance, such as Palm and News, our method maintains a positive margin (\eg $+0.25\%$ to $+1.10\%$ in Accuracy). The narrowing of the ablation gap on these datasets confirms that in highly separable spaces, a static graph is often sufficient; yet, our adaptive mechanism safely refines boundaries without causing structural deterioration.
\end{itemize}

By shifting the computational burden of neighbor voting and graph optimization to the training phase, the $k$NN-Graph effectively decouples inference latency from model complexity. The results confirm that the performance superiority is directly attributable to the learned adaptive topology, validating our central hypothesis that data-driven graph construction fundamentally resolves the efficiency-accuracy trade-off inherent in approximate nearest neighbor search.

\subsection{Algorithmic Acceleration and Training Overhead}
\label{sec:training_time}

While the proposed $k$NN-Graph framework achieves near-instantaneous, logarithmic-time retrieval during the inference phase, this efficiency is secured by fundamentally shifting the computational burden to the offline training phase. To provide a holistic and transparent view of the algorithmic overhead, we explicitly detail the optimization dynamics, the exact training times, and the computational environment used in our empirical evaluations.

Optimization Acceleration: A hallmark of our framework's practical viability is its extreme training efficiency. To solve the $\ell_1$-norm regularized self-representation problem (Eq. (\ref{eq:objective})), our implementation abandons standard coordinate descent in favor of an accelerated Fast Iterative Shrinkage-Thresholding Algorithm (FISTA) augmented with a backtracking line search mechanism. This guarantees that the objective function bypasses the slow asymptotic descent phase, achieving precipitous, cliff-like convergence (based on a rigorous mixed relative-absolute error tolerance) typically within merely 5 to 10 iterations.

Training Latency and Environment: Table \ref{tab:training_time} reports the average offline training time required to construct the $k$NN-Graph across all benchmark datasets over the 10-fold cross-validation trials. This recorded duration comprehensively accounts for the FISTA-accelerated kernelized representation, local density estimation, and HNSW index construction. Furthermore, the utilized programming language and precise hardware configurations are explicitly documented.

As observed in Table \ref{tab:training_time}, the training latency scales optimally with the inherent topological complexity of the datasets. For instance, while structured shape data (Mpeg) requires under a minute, the dense and highly overlapping manifolds of the Caltech dataset necessitate approximately 2.4 hours to precisely capture the adaptive connectivity. However, it is imperative to emphasize that this optimization is a strictly one-time, offline computation. Once the topological graph and consensus labels are precomputed, the model entirely circumvents the severe runtime latency associated with traditional distance calculations and voting aggregations, unlocking highly scalable, real-time query deployment.

\subsection{Optimization Stability and Convergence}

To empirically validate the robustness and computational efficiency of the proposed Kernelized Self-Representation module, we analyzed the optimization dynamics of the objective function during the training phase. Fig.~\ref{fig5} illustrates the iterative evolution of the objective values across the six benchmark datasets.

A hallmark of our framework's practical viability is its extreme training efficiency. To solve the $\ell_1$-norm regularized optimization problem, our implementation leverages an accelerated FISTA augmented with a backtracking line search mechanism. As vividly demonstrated in Fig.~\ref{fig4}, this implementation upgrade completely bypasses the slow asymptotic descent phase typical of standard gradient methods, yielding distinct and highly efficient optimization behaviors:

\begin{itemize}
	\item Precipitous Convergence Regime: For datasets such as Binalpha, Caltech, Corel, Mpeg, and Palm, the objective function exhibits a cliff-like drastic descent, reaching a strict stable plateau within merely 2 to 3 iterations. This instantaneous convergence validates that the dynamic backtracking line search effectively identifies the optimal large step sizes, allowing the algorithm to rapidly locate the optimal sparse subspace support.
	\item Accelerated Monotonic Descent: For high-dimensional and sparse datasets such as News, the objective function exhibits a rapid, monotonic decrease, reaching convergence within 8 iterations. This indicates that even in highly complex optimization landscapes, the Nesterov momentum intrinsically accelerates the convergence rate without being trapped in local instabilities.
\end{itemize}

Critically, the convergence trajectories across all six datasets display absolute monotonicity without any oscillatory divergence (ripples) in the later stages. This strict stability perfectly aligns with the theoretical $\mathcal{O}(1/k^2)$ convergence bounds of accelerated proximal gradient methods applied to convex Lasso-type problems. It confirms that the combination of the adaptive regularization parameter $\lambda^{(j)}$ and the line search mechanism safely regularizes the momentum trajectory. These findings provide compelling empirical evidence that the offline training phase of $k$NN-Graph is mathematically rigorous, highly scalable, and computationally reliable, establishing a solid foundation for real-time graph inference.

\section{Methods}\label{sec:methods}

We present the Adaptive Graph Model for $k$NN ($k$NN-Graph), a framework designed to resolve the fundamental trade-off between inference efficiency and classification accuracy. Unlike traditional approaches that perform expensive neighbor searches during inference, our method introduces a paradigm shift by decoupling the computational complexity of neighborhood determination from the query phase. This is achieved through two tightly integrated components: a Kernelized Self-Representation Module that adaptively learns the optimal local topology during training, and a Hierarchical Graph Indexing Module that encodes these decisions into a navigable structure for logarithmic-time retrieval.

\subsection{Notations and Problem Formulation}
Let $\mathbf X = [\mathbf x_1, \ldots, \mathbf x_n] \in \mathbb{R}^{d \times n}$ denote the training dataset containing $n$ samples, where each $\mathbf x_j \in \mathbb{R}^d$ is a $d$-dimensional feature vector. The goal is to learn a sparse adjacency matrix $\mathbf W \in \mathbb{R}^{n \times n}$ where each non-zero element $w_{ij}$ represents the directed connection weight from sample $i$ to sample $j$. Crucially, we do not fix the number of non-zero elements per column; instead, the sparsity level (representing the optimal neighbor count $k$) is learned adaptively. Key notations used throughout this paper are summarized in the Supplementary Information (Table~S2).

\subsection{Adaptive Neighborhood Learning via Kernelized Self-Representation}
\label{sec:adaptive_learning}

To capture the complex, nonlinear manifold structure inherent in high-dimensional real-world data, we move beyond linear subspace assumptions. We propose a kernelized self-representation model that reconstructs each data point as a sparse linear combination of other points in a Reproducing Kernel Hilbert Space (RKHS).

\subsubsection{The Kernelized Optimization Objective}
For each training sample $\mathbf x_j$, we seek a sparse coefficient vector $\mathbf w_j$ by minimizing the reconstruction error combined with an adaptively weighted $\ell_1$-norm sparsity constraint. The optimization problem is formulated as:

\begin{equation}
\label{eq:objective}
\min_{\mathbf w_j} \left\| \mathbf K(\cdot, j) - \mathbf K \mathbf w_j \right\|_2^2 + \lambda^{(j)} \|\mathbf w_j\|_1, \quad \text{s.t. } w_{jj} = 0,
\end{equation}
where $\mathbf K \in \mathbb{R}^{n \times n}$ is a composite kernel matrix designed to integrate both feature-based proximity and label consistency. It is defined as:
\begin{equation}
\label{eq:kernel}
\mathbf K_{ij} = \alpha \exp \left( - \frac{\|\mathbf x_i - \mathbf x_j\|^2}{2\sigma^2} \right) + (1 - \alpha) K_{\text{class}}(i,j).
\end{equation}

Here, $K_{\text{class}}(i,j)$ takes the value 1 if $\mathbf x_i$ and $\mathbf x_j$ share the same class label, and a discount factor $\gamma$ (set to 0.1) otherwise. This composite kernel ensures that the learned neighbors are not only geometrically close but also semantically consistent. 

Importantly, while the mapping to the RKHS is global, this formulation explicitly accommodates the inherent non-homogeneity of real-world data manifolds. By integrating the local geometric proximity and semantic consistency in the composite kernel, and coupling it with the density-aware regularization (introduced next), our framework adaptively modulates the local representation. This ensures that the global RKHS assumption does not limit data diversity, but rather allows for the construction of a highly localized and heterogeneous topological graph.

\subsubsection{Density-Aware Adaptive Regularization}
A critical innovation of our framework is the dynamic determination of the regularization parameter $\lambda^{(j)}$. A fixed $\lambda$ would impose a uniform sparsity level across the dataset, ignoring local density variations. To address this, we introduce a density-aware mechanism.

We first compute a multi-scale local density estimate $\rho(j)$ for each sample $\mathbf x_j$ by averaging the inverse distances over various neighborhood scales. It is crucial to note that $\rho(j)$ is calculated a priori using the initial spatial (Euclidean) distances in the original feature space, entirely independent of the topological weights $\mathbf w_j$. This sequential independence eliminates any circular dependency, providing a deterministic prior to initialize the regularization parameter $\lambda^{(j)}$ before solving Eq. (\ref{eq:objective}). Based on this estimate, the adaptive regularization parameter is derived as:
\begin{equation}
\label{eq:lambda}
\lambda^{(j)} = \lambda_{\min} + (\lambda_{\max} - \lambda_{\min}) \cdot (1 - \rho(j)).
\end{equation}

This formulation enforces a simpler model (stronger sparsity, higher $\lambda$) in low-density regions to prevent overfitting, while allowing a more complex model (weaker sparsity, lower $\lambda$) in high-density regions to capture fine-grained local structures. The optimization of Eq. (\ref{eq:objective}) is efficiently solved via an accelerated FISTA augmented with a backtracking line search, as detailed in Algorithm \ref{alg:adaptive_learning}.

\subsection{Hierarchical Graph Indexing with Precomputed Decisions}
\label{sec:graph_indexing}

The solution to Eq. (\ref{eq:objective}) yields a sparse matrix $\mathbf W$, where the non-zero entries in column $j$ define the optimal neighbor set $N^*(j) = \{i : w_{ij} \neq 0\}$ and the optimal neighbor count $K_j = |N^*(j)|$. Unlike conventional methods that stop at neighborhood identification, our framework proceeds to construct a specialized index.

\subsubsection{Embedding Precomputed Intelligence into HNSW}
We construct an HNSW graph, but with a fundamental architectural modification. Instead of storing only raw data, each node in our graph is an intelligent container storing the precomputed classification result.

For every training node $j$, we compute a weighted consensus label $\hat{y}_j$ during the training phase:
\begin{equation}
\label{eq:consensus}
\hat{y}_j = \operatorname{weighted\_mode}\left( \{(y_j, w_{\text{self}})\} \cup \{(y_i, w_{ij}) \mid i \in N^*(j)\} \right),
\end{equation}
where $w_{\text{self}}$ is a strong self-reinforcement weight, strictly set to twice the maximum neighbor weight ($w_{\text{self}} = 2 \max_{i \in N^*(j)} w_{ij}$).

This consensus mechanism is fundamentally designed not just to aggregate labels, but to act as an implicit local denoising filter. Because of the dominant self-weight, a training node retains its correct ground-truth label unless it is overwhelmingly surrounded by strong connections to an opposing class. Consequently, if an erroneous consensus decision occurs (\ie $\hat{y}_j \neq y_j$), it effectively identifies and smoothes over noisy labels or isolated outlier points deep within another class's distribution. The query thus inherits a robust, smoothed decision boundary, significantly enhancing the model's resilience to training noise.

The HNSW structure is then built using these enriched nodes. The hierarchy consists of multiple layers where lower layers encode the precise, learned connectivity $N^*(j)$, and upper layers provide long-range links for logarithmic-time navigation.

\subsubsection{Logarithmic-Time Inference Mechanism}
During the inference phase, the computational complexity is strictly decoupled from the training set size $n$ for the purpose of classification logic. For a query $q$:
1. The system traverses the HNSW graph starting from the top layer to locate the nearest node $\mathbf x_{nearest}$ (Algorithm \ref{alg:query}).
2. Upon reaching this node, the system immediately retrieves the stored precomputed label $\hat{y}_{nearest}$.
3. This label is returned as the final prediction.

This design eliminates the need for calculating distances to $k$ neighbors or performing a voting process at runtime, effectively reducing the classification operation to a graph lookup.

\begin{algorithm}[H]
	\caption{Adaptive Neighborhood Learning (Training Phase)}
	\label{alg:adaptive_learning}
	\begin{algorithmic}[1]
		\Require Training data $\mathbf X$, Kernel parameters $\alpha, \sigma$, Regularization bounds $\lambda_{\min}, \lambda_{\max}$
		\Ensure Adaptive neighbor weights $\mathbf W$
		\State Compute composite kernel matrix $\mathbf K$ using Eq. (\ref{eq:kernel})
		\For{$j = 1$ to $n$}
		\State Compute multi-scale local density $\rho(j)$ a priori using initial spatial distances
		\State Determine adaptive regularization $\lambda^{(j)}$ using Eq. (\ref{eq:lambda})
		\State Solve $\mathbf{w}_j = \arg\min \left\| \mathbf{K}{(\cdot, j)} - \mathbf{K} \mathbf{w}_j \right\|_2^2 + \lambda^{(j)} \|\mathbf{w}_j\|_1$ via FISTA with backtracking line search
		\State Normalize $\mathbf w_j$ and store as column $j$ in $\mathbf W$
		\EndFor
	\end{algorithmic}
\end{algorithm}

\subsection{Algorithmic Implementation and Workflow}
The comprehensive implementation of the Adaptive-GkNN framework is formalized in the following algorithms. Algorithm \ref{alg:adaptive_learning} details the training phase, where the kernelized self-representation model learns the optimal local topology. Algorithm \ref{alg:query} integrates the graph construction with the inference logic, demonstrating how precomputed decisions enable rapid retrieval. To provide an intuitive overview of these interacting components, a schematic flowchart spanning the kernel mapping, HNSW construction, and inference stages is presented in Fig. \ref{fig5}.

\begin{algorithm}[H]
	\caption{Graph Construction and Inference Process}
	\label{alg:query}
	\begin{algorithmic}[1]
		\State Phase 1: Precomputation \& Indexing
		\Require Self-representation matrix $\mathbf W$, Training labels $\mathbf Y$
		\For{$j = 1$ to $n$}
		\State Identify neighbor set $N^*(j) = \{i : w_{ij} \neq 0\}$
		\State Compute consensus label $\hat{y}_j$ using Eq. (\ref{eq:consensus})
		\State Node $n_j \gets \text{CreateNode}(\text{Feature}=\mathbf x_j, \text{Label}=\hat{y}_j)$
		\State \text{HNSW\_Insert}(Index, $n_j$, Neighbors=$N^*(j)$)
		\EndFor
		
		\State Phase 2: Inference
		\Require Query point $q$, HNSW Index
		\Function{Predict}{$q$}
		\State $node^* \gets \text{HNSW\_Search}(q, \text{EntryPoints})$ \Comment{Logarithmic-time navigation}
		\State \Return $node^*. \hat{y}$ \Comment{Return precomputed label directly}
		\EndFunction
	\end{algorithmic}
\end{algorithm}

\subsection{Theoretical Guarantees}

We provide theoretical insights into the structural stability of the learned graph. Ensuring that the sparse learning process does not result in a fragmented graph is critical, as the navigability of the HNSW index relies on the connectivity of the underlying topology. Disconnected components or isolated nodes would trap the greedy search algorithm in local optima, preventing the query from reaching its true nearest neighbor.

\begin{proposition}[Connectivity Stability]
	Let $\mathbf K$ be the kernel matrix and $\mathbf k_j$ be its $j$-th column. If the minimum regularization parameter satisfies $\lambda_{\min} < \min_j \|\mathbf K(\cdot, j)\|_2^2$, then every sample $\mathbf x_j$ is guaranteed to have at least one neighbor ($w_{ij} \neq 0$ for some $i$), preventing the emergence of isolated nodes.
\end{proposition}

\begin{proof}
	(Sketch) Consider the objective function for sample $\mathbf x_j$. The trivial zero solution $\mathbf w_j = \mathbf 0$ yields a reconstruction cost equal to the squared norm of the feature vector in RKHS, \ie $\|\mathbf k_j\|_2^2$.
	Consider activating a connection to the most similar sample $i$ (where the kernel similarity $K_{ij}$ is maximal). This reduces the reconstruction error significantly.
	Provided that the regularization penalty $\lambda^{(j)}$ is strictly smaller than this reduction in reconstruction error (which is enforced by the bound $\lambda_{\min} < \min_j \|\mathbf k_j\|_2^2$), the optimization objective will always achieve a lower value with a non-zero weight than with the zero vector.
	Consequently, the optimal solution must contain at least one non-zero element, guaranteeing that $\mathbf x_j$ is connected to the graph.
\end{proof}

\subsection{Complexity Analysis}
\label{sec:complexity}

To theoretically quantify the efficiency gains of the proposed $k$NN-Graph framework, we analyze the time complexity of the inference phase and compare it with traditional $k$NN and standard ANN methods, as summarized in Table \ref{tab:complexity}.

Let $n$ be the number of training samples, $d$ be the feature dimensionality, $k$ be the number of required neighbors for consensus, and $\bar{m}$ be the average degree of the graph nodes.

Traditional $k$NN: The brute-force $k$NN algorithm computes the distance between the query and every training sample, followed by a sorting or selection operation to identify the nearest neighbors. The complexity is bounded by:
\begin{equation}
T_{\text{$k$NN}} = \mathcal{O}(n \cdot d + n \log k).
\end{equation}

This strict linear dependence on $n$ creates a fundamental scalability bottleneck for massive datasets.

Tree-based ANN (\eg KD-Tree): While spatial partitioning structures like KD-Trees can achieve logarithmic search time in low dimensions, they severely suffer from the curse of dimensionality. In high-dimensional feature spaces, the search algorithm is forced to backtrack across numerous spatial boundaries, causing the retrieval complexity to rapidly degrade to $\mathcal{O}(n \cdot d)$. Consequently, tree-based methods essentially degenerate into brute-force search for high-dimensional real-world data.

Standard Graph-based ANN (\eg HNSW): These methods utilize a hierarchical graph structure to navigate to the candidate region in logarithmic time. However, to execute a $k$NN classification, they must maintain a dynamic candidate pool, retrieve the top $k$ nodes, and subsequently perform a majority voting or ranking operation. The complexity is approximately:
\begin{equation}
T_{\text{ANN}} = \mathcal{O}(\bar{m} \log n \cdot d + k \log k).
\end{equation}

The term $\mathcal{O}(\bar{m} \log n \cdot d)$ accounts for the graph traversal and exact distance evaluations, while $\mathcal{O}(k \log k)$ reflects the algorithmic overhead of ranking the candidates and voting for the final discrete label.

Proposed $k$NN-Graph: Our method identically leverages the hierarchical graph for navigation but fundamentally eliminates the post-retrieval aggregation. Since the topological connectivity and consensus predictions are adaptively embedded during the offline training phase, the search only needs to route to the single nearest node ($k=1$ search configuration). Upon reaching this node, the precomputed consensus label is retrieved in $\mathcal{O}(1)$ time. The inference complexity is strictly limited to the graph traversal:
\begin{equation}
T_{\text{Ours}} = \mathcal{O}(\bar{m} \log n \cdot d).
\end{equation}

Crucially, our method completely decouples the inference latency from the neighborhood size $k$. In scenarios requiring large $k$ values (\eg to ensure robustness in dense or highly noisy regions), standard ANN methods inevitably suffer from an expanded search space overhead and increased voting costs. By proactively shifting this computational burden to the offline training phase, our approach renders the online inference cost strictly independent of the decision complexity, achieving pure, hardware-friendly logarithmic retrieval.

\section*{Data Availability}
The datasets generated and analyzed during the current study have been deposited in the Zenodo database under accession code 10.5281/zenodo.19531320 [\url{https://doi.org/10.5281/zenodo.19531320}].

\section*{Code Availability}
The custom code, algorithms, and scripts used to implement the $k$NN-Graph framework and perform the analyses reported in this study are publicly available in the GitHub repository at \url{https://github.com/Lijy207/kNN-Graph}. To ensure long-term reproducibility, a permanent version of the code at the time of publication has been deposited in Zenodo under the DOI: 10.5281/zenodo.20236665 [\url{https://doi.org/10.5281/zenodo.20236665}].

\bibliography{refs}

\section*{Acknowledgments}
The authors acknowledge Zhejiang University, Central South University, and Guangxi Normal University for providing the institutional support and infrastructure necessary to conduct this research.

\section*{Funding}
This work was supported by the Postdoctoral Fellowship Program of CPSF (Grant No. GZC20251062) and the China Postdoctoral Science Foundation (Grant No. 2025M781521).

\section*{Author Contributions}
S.Z. conceived the project and secured the funding. S.Z. supervised the research. J.L. and H.X. developed and trained the $k$NN-Graph framework, and performed the formal analysis. J.L. wrote the initial draft of the manuscript. All authors (J.L., H.X., and S.Z.) discussed the results and contributed to the writing, review, and editing of the final manuscript.

\section*{Competing Interests}
The authors declare no competing interests.

\newpage

\begin{table}[!ht]
	\centering
	\caption{Average classification accuracy over 10 experimental runs (\%). The best-performing results for each dataset are highlighted in bold.}
	\label{tab:accuracy}
	\begin{tabular}{lccccccc}
		\toprule
		Method & Binalpha & Caltech & Corel & Mpeg & News & Palm & Mean \\
		\midrule
		CV-$k$NN   & 66.74 & 58.25 & 32.50 & 80.14 & 92.87 & 99.65 & 71.69 \\
		E$k$NN     & 61.40 & 59.84 & 31.12 & 67.71 & 87.13 & 99.25 & 67.74 \\
		KD-Tree    & 65.74 & 56.98 & 32.54 & 73.21 & 92.92 & 99.60 & 70.15 \\
		OWA$k$NN   & 68.38 & 60.57 & 32.92 & 83.64 & 91.84 & 99.70 & 72.84 \\
		PL$k$NN    & 65.23 & 52.09 & 32.56 & 82.29 & 69.47 & 98.90 & 66.76 \\
		XHMA$k$NN  & 67.87 & 58.93 & 32.80 & 72.64 & 63.50 & 99.75 & 65.92 \\
		$k$*tree   & 66.81 & 59.41 & 32.52 & 80.21 & 92.59 & 99.55 & 71.85 \\
		O$k$NN     & 67.10 & 59.18 & 32.60 & 81.21 & 93.50 & 99.70 & 72.22 \\
		\textbf{$k$NN-Graph} & \textbf{69.65} & \textbf{61.09} & \textbf{33.62} & \textbf{84.57} & \textbf{93.80} & \textbf{99.80} & \textbf{73.76} \\
		\bottomrule
	\end{tabular}
\end{table}

\newpage

\begin{table}[!ht]
	\centering
	\caption{Average macro-precision over 10 experimental runs (\%). The best-performing results for each dataset are highlighted in bold.}
	\label{tab:macro_precision}
	\begin{tabular}{lccccccc}
		\toprule
		Method & Binalpha & Caltech & Corel & Mpeg & News & Palm & Mean \\
		\midrule
		CV-$k$NN   & 70.80 & 51.29 & 33.37 & 81.87 & 92.97 & 99.77 & 71.68 \\
		E$k$NN     & 64.57 & 53.94 & 28.68 & 61.86 & 89.83 & 99.50 & 66.40 \\
		KD-Tree    & 73.67 & 49.64 & \textbf{35.92} & 70.87 & 92.99 & 99.72 & 70.47 \\
		OWA$k$NN   & 72.16 & 52.95 & 34.26 & 84.58 & 91.94 & 99.80 & 72.62 \\
		PL$k$NN    & 70.29 & 45.61 & 34.25 & 83.66 & 75.88 & 99.08 & 68.13 \\
		XHMA$k$NN  & 74.00 & 53.39 & 35.75 & 70.07 & 80.61 & 99.82 & 68.94 \\
		$k$*tree   & 73.16 & 53.03 & 34.04 & 80.48 & 92.68 & 99.70 & 72.18 \\
		O$k$NN     & 72.53 & 52.78 & 33.54 & 81.45 & 93.59 & 99.80 & 72.28 \\
		\textbf{$k$NN-Graph} & \textbf{74.97} & \textbf{54.91} & 34.56 & \textbf{85.70} & \textbf{93.87} & \textbf{99.87} & \textbf{73.98} \\
		\bottomrule
	\end{tabular}
\end{table}

\newpage

\begin{table}[!ht]
	\centering
	\caption{Average inference time (in seconds). The best-performing (shortest) results for each dataset are highlighted in bold.}
	\label{tab:inference_time}
	\begin{tabular}{lccccccc}
		\toprule
		Method & Binalpha & Caltech & Corel & Mpeg & News & Palm & Mean \\
		\midrule
		CV-$k$NN   & 1.5054 & 9.0065 & 5.7486 & 25.5809 & 3409.1249 & 1.5079 & 575.4124 \\
		E$k$NN     & 1.0353 & 86.1351 & 2.4017 & 6.6196 & 18.4378 & 5.1370 & 19.9611 \\
		KD-Tree    & 1.5037 & 56.6143 & 16.7587 & 7.5409 & 95.0954 & 3.7753 & 30.2147 \\
		OWA$k$NN   & 2.4849 & 43.9984 & 9.9134 & 62.7050 & 188.2357 & 0.8738 & 51.3685 \\
		PL$k$NN    & 0.2833 & 9.5861 & 6.1142 & 9.6319 & 116.6768 & 0.4947 & 23.7978 \\
		XHMA$k$NN  & 0.1983 & 7.0775 & 2.7548 & 3.2129 & 44.6574 & 0.2523 & 9.6922 \\
		$k$*tree   & 0.0404 & 0.6099 & 0.1705 & 0.1408 & 10.7403 & 0.1208 & 1.9705 \\
		O$k$NN     & 0.0302 & 0.5353 & 0.1511 & 0.0601 & 0.5618 & 0.1108 & 0.2370 \\
		\textbf{$k$NN-Graph} & \textbf{0.0022} & \textbf{0.0097} & \textbf{0.0094} & \textbf{0.0573} & \textbf{0.5233} & \textbf{0.0024} & \textbf{0.1007} \\
		\bottomrule
	\end{tabular}
\end{table}

\newpage

\begin{table}[htbp]
	\centering
	\caption{Average offline training time of the proposed $k$NN-Graph framework, along with the programming language and hardware specifications.}
	\label{tab:training_time}
	\begin{tabular}{lrl|l} 
		\toprule
		Dataset & Training Time (s) & \hspace{1em} & Language \& Hardware Specifications \\
		\midrule
		Binalpha & 18.47 \hspace{1.5em} && Language: MATLAB R2024b \\
		Caltech  & 8635.40 \hspace{1.5em} && OS: Windows 10 \\
		Corel    & 1672.80 \hspace{1.5em} && CPU: AMD Ryzen 9 3900X 12-Core (3.79 GHz) \\
		Mpeg     & 47.65 \hspace{1.5em}   && RAM: 32.0 GB \\
		News     & 1009.56 \hspace{1.5em} && GPU: NVIDIA GeForce GTX 1060 (6 GB) \\
		Palm     & 98.50 \hspace{1.5em}   && Storage: 4 TB Disk \\
		\bottomrule
	\end{tabular}
\end{table}

\newpage

\begin{table}[htbp]
	\centering
	\caption{Comparison of Inference Time Complexity.}
	\label{tab:complexity}
	\begin{tabular}{lccc}
		\toprule
		Method & \makecell[c]{Search \\ Complexity} & \makecell[c]{Voting/Ranking \\ Cost} & \makecell[c]{Total \\ Complexity} \\
		\midrule
		Brute-force $k$NN & $\mathcal{O}(n \cdot d)$ & $\mathcal{O}(n \log k)$ & $\mathcal{O}(n \cdot d)$ \\
		\addlinespace 
		KD-Tree (High Dim.) & $\mathcal{O}(n \cdot d)$ & $\mathcal{O}(k \log k)$ & $\mathcal{O}(n \cdot d)$ \\
		\addlinespace
		Standard HNSW & $\mathcal{O}(\bar{m} \log n \cdot d)$ & $\mathcal{O}(k \log k)$ & $\mathcal{O}(\bar{m} \log n \cdot d + k \log k)$ \\
		\addlinespace
		$k$NN-Graph (Ours) & $\mathcal{O}(\bar{m} \log n \cdot d)$ & $0$ & $\mathcal{O}(\bar{m} \log n \cdot d)$ \\
		\bottomrule
	\end{tabular}
\end{table}

\newpage

\begin{figure}[!ht]
	\begin{center}
		\vspace{-1mm}
		\subfigure{\scalebox{0.4}{\includegraphics{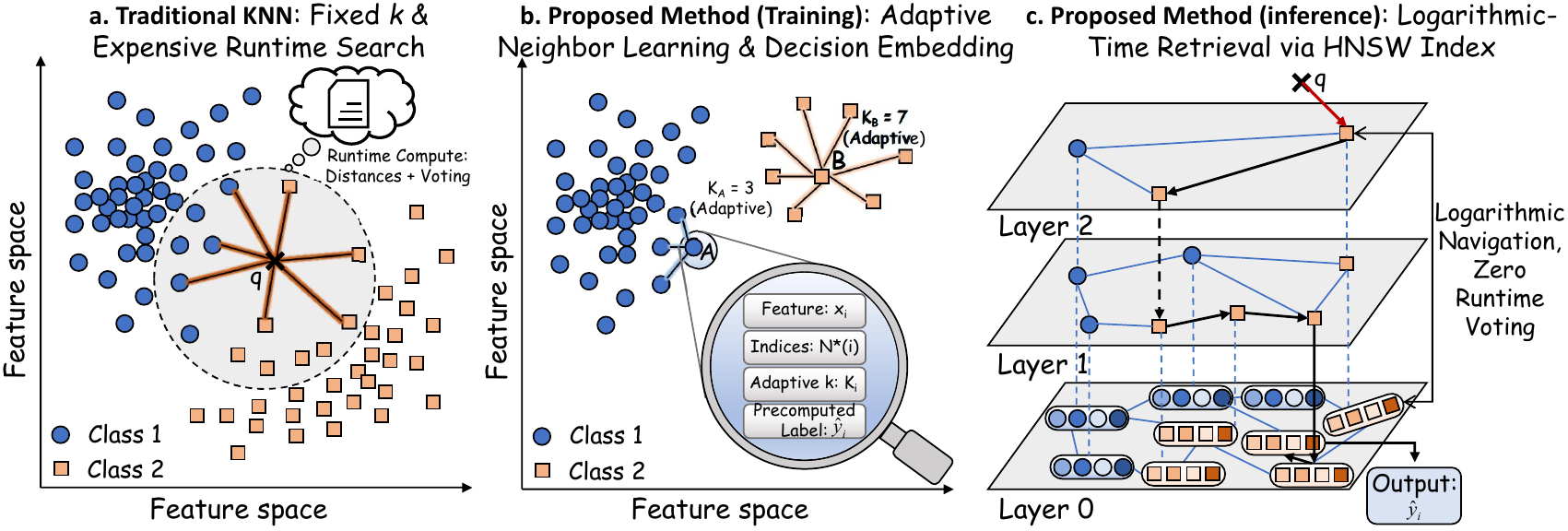}}}		
		\caption{\textbf{Conceptual comparison between traditional $k$NN and the adaptive graph model.} \textbf{a} Traditional $k$NN classification. The standard algorithm applies a global, fixed neighborhood size ($k$) around a query point ($q$), requiring runtime distance computations (indicated by orange solid lines) and majority voting. Blue circles and orange squares represent Class 1 and Class 2 samples, respectively. \textbf{b} Adaptive offline training. The proposed method evaluates local data density to dynamically assign a node-specific adaptive $k$ (for example, $K_\mathrm{A}=3$ for node A and $K_\mathrm{B}=7$ for node B). A consensus label ($\hat{y}_i$) is pre-computed and structurally embedded into each node along with its feature vector ($\mathbf{x}_i$), neighbor indices ($N^*(i)$), and the assigned adaptive $k$ ($K_i$), as detailed in the magnifying glass inset. \textbf{c} Logarithmic-time inference. Utilizing a Hierarchical Navigable Small World (HNSW) index comprising multiple layers (Layer 0 to Layer 2), an online query ($\mathbf{x}_q$) routes through descending graph layers (indicated by the red entry arrow and black routing arrows) to locate the nearest pre-computed node. The embedded label ($\hat{y}_i$) is directly retrieved as the final output, removing the requirement for real-time distance sorting and voting. }
		\label{fig1}
	\end{center}

\end{figure}

\begin{figure*}[!ht]
	\begin{center}
		\vspace{-1mm}
		\subfigure{\scalebox{0.39}{\includegraphics{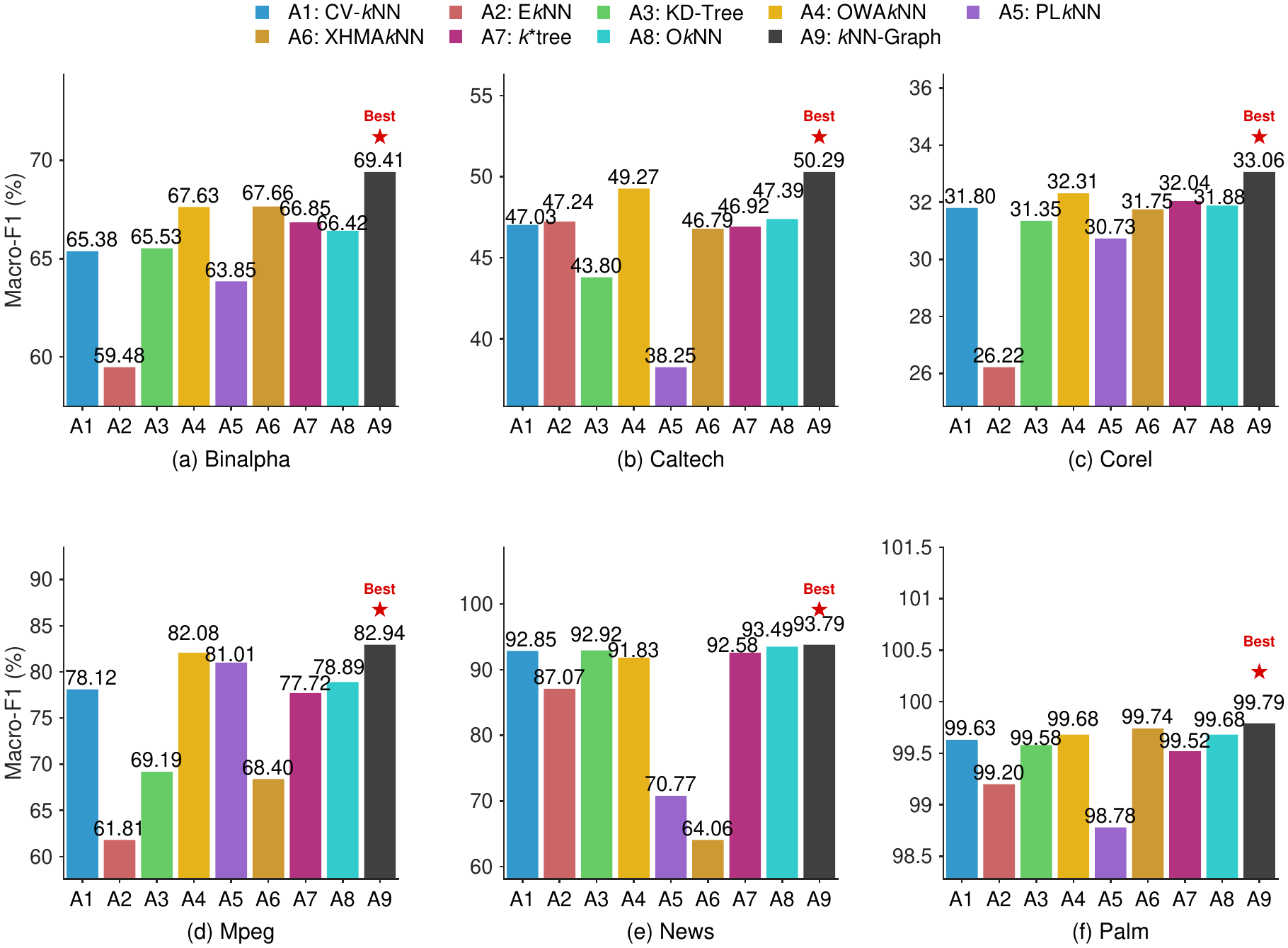}}}	
		
		\caption{\textbf{Algorithmic performance comparison using Average Macro-F1 Score.} \textbf{a}--\textbf{f}, Comparison of Average Macro-F1 Score (\%) between the proposed $k$NN-Graph framework (A9, dark grey bars) and eight established baseline algorithms (A1--A8) across six diverse datasets: Binalpha (\textbf{a}), Caltech (\textbf{b}), Corel (\textbf{c}), Mpeg (\textbf{d}), News (\textbf{e}), and Palm (\textbf{f}). Bar heights and the numerical values displayed above each bar represent the mean Macro-F1 Score computed over ten independent runs (per the common title). For baseline algorithms A1--A8, varied colors mirror the color coding provided in the main legend. In each panel, top performance is marked by a red star ($\star$) and labelled Best.}	
		\label{fig2}
	\end{center}
\end{figure*} 

\begin{figure*}[!ht]
	\begin{center}
		\vspace{-1mm}
		\subfigure{\scalebox{0.35}{\includegraphics{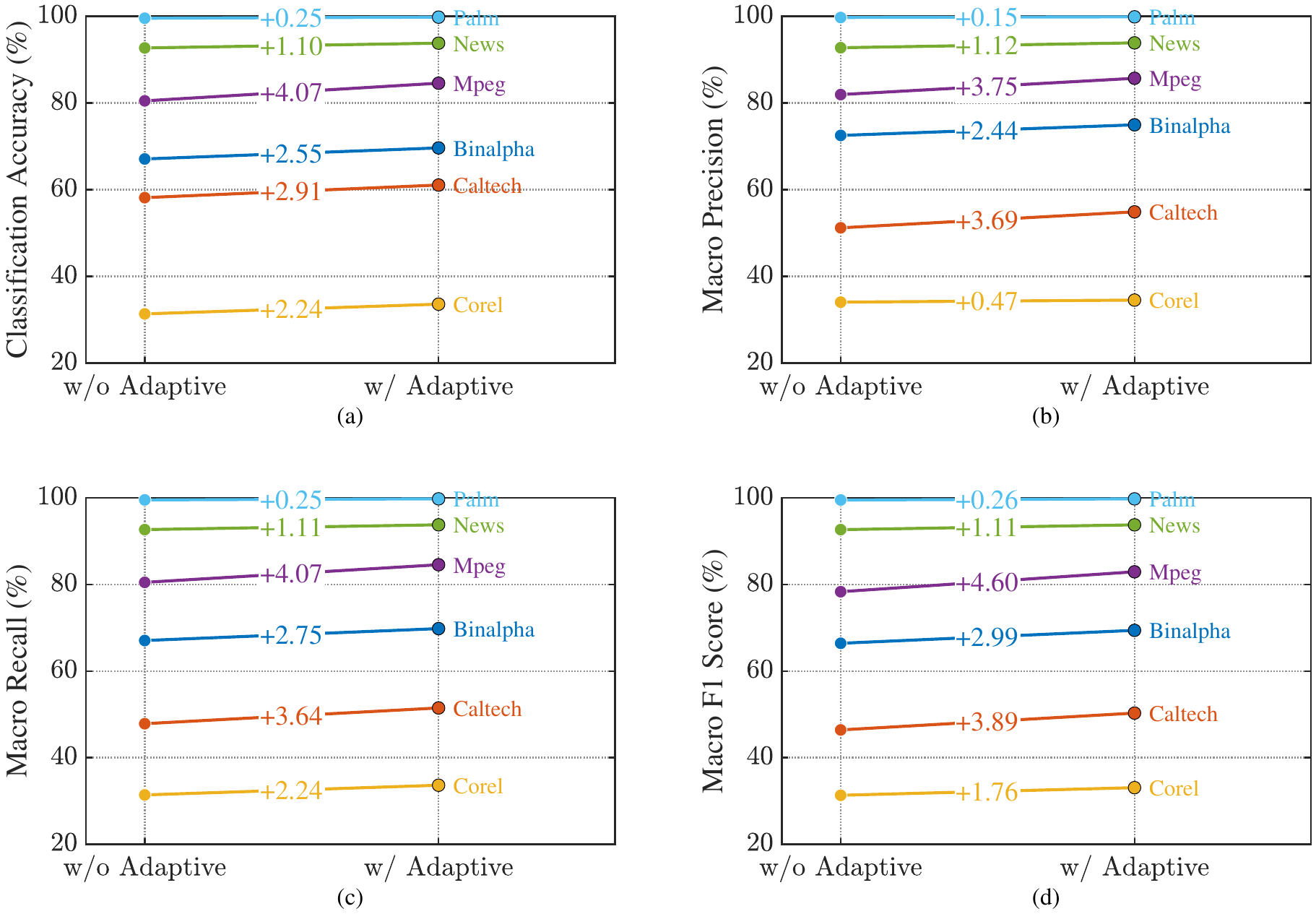}}}	
		
		\caption{\textbf{Ablation analysis of the Adaptive Neighborhood Learning mechanism.} \textbf{a}--\textbf{d}, Pairwise performance comparisons between the baseline static Hierarchical Navigable Small World (HNSW) graph (w/o Adaptive) and the proposed method (w/ Adaptive) across six benchmark datasets. The panels present classification accuracy (\textbf{a}), macro precision (\textbf{b}), macro recall (\textbf{c}), and macro F1-score (\textbf{d}). Colours denote the specific benchmark datasets as labelled. Values annotated on the connecting lines indicate the absolute percentage point gains ($+\Delta$) for each dataset. The positive shifts across all metrics indicate that jointly optimizing adaptive neighbor counts and dynamic weights enhances the model's discriminative capacity over a static graph structure.}
		\label{fig3}	
	\end{center}
	
\end{figure*} 

\begin{figure*}[!ht]
	\begin{center}
		\vspace{-1mm}
		\subfigure{\scalebox{0.23}{\includegraphics{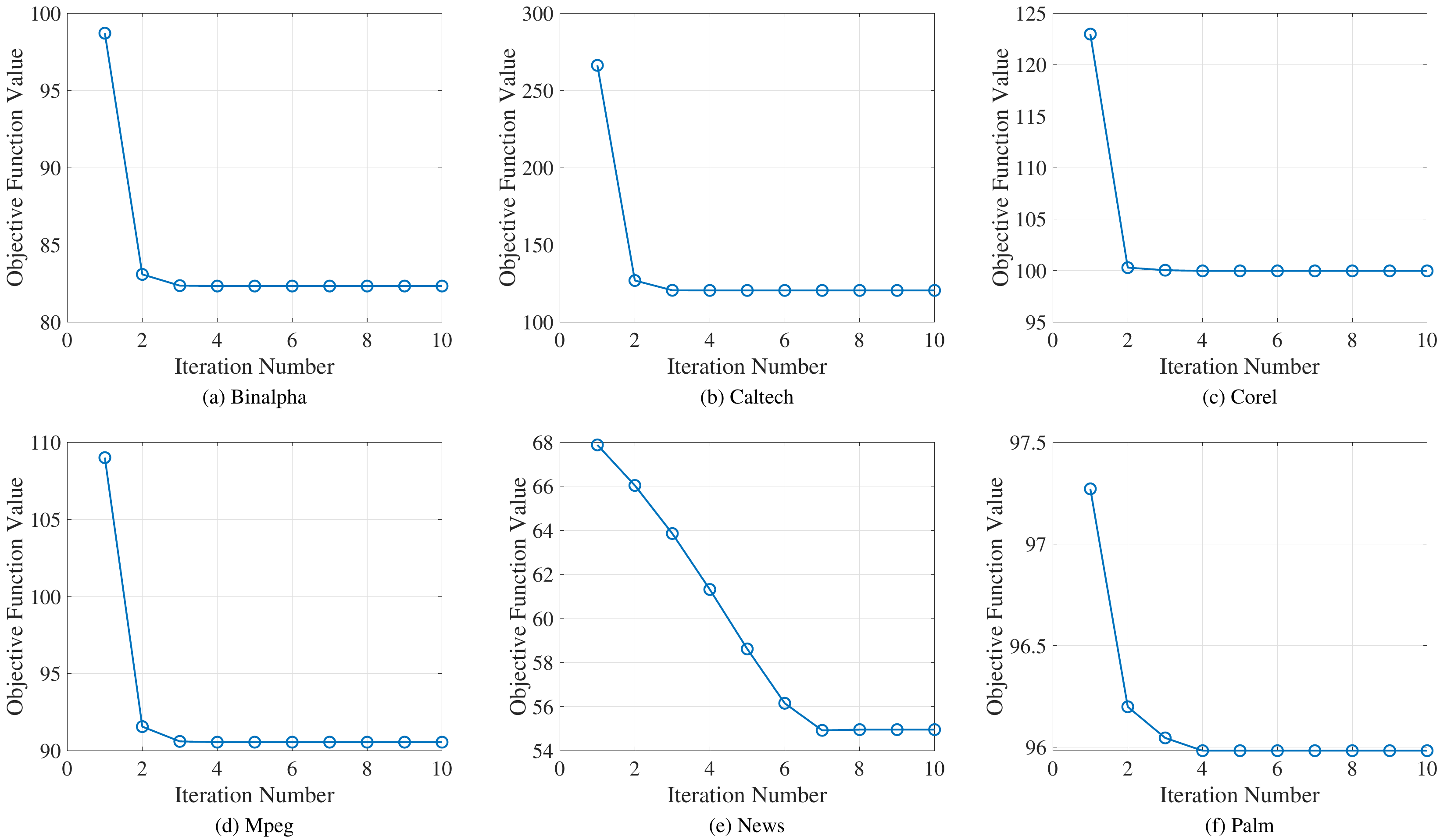}}}

		\caption{\textbf{Convergence analysis of the proposed objective function.} \textbf{a}--\textbf{f}, Evolution of the objective function value across successive iterations for the proposed method on six benchmark datasets: Binalpha (\textbf{a}), Caltech (\textbf{b}), Corel (\textbf{c}), Mpeg (\textbf{d}), News (\textbf{e}), and Palm (\textbf{f}). The blue lines with circular markers trace the optimization trajectory. For the majority of the evaluated datasets (\textbf{a}--\textbf{d} and \textbf{f}), the objective function exhibits a rapid initial reduction, achieving stable convergence within the first three to four iterations. Notably, for the high-dimensional and sparse News dataset (\textbf{e}), the objective function demonstrates a consistent, monotonic decrease and strictly converges within 8 iterations. These distinct trajectories objectively confirm the stable convergence properties of the proposed method across varying data complexities.}	
		\label{fig4}
	\end{center}
\end{figure*}

\begin{figure*}[ht]
	\centering
	\includegraphics[width=1\textwidth]{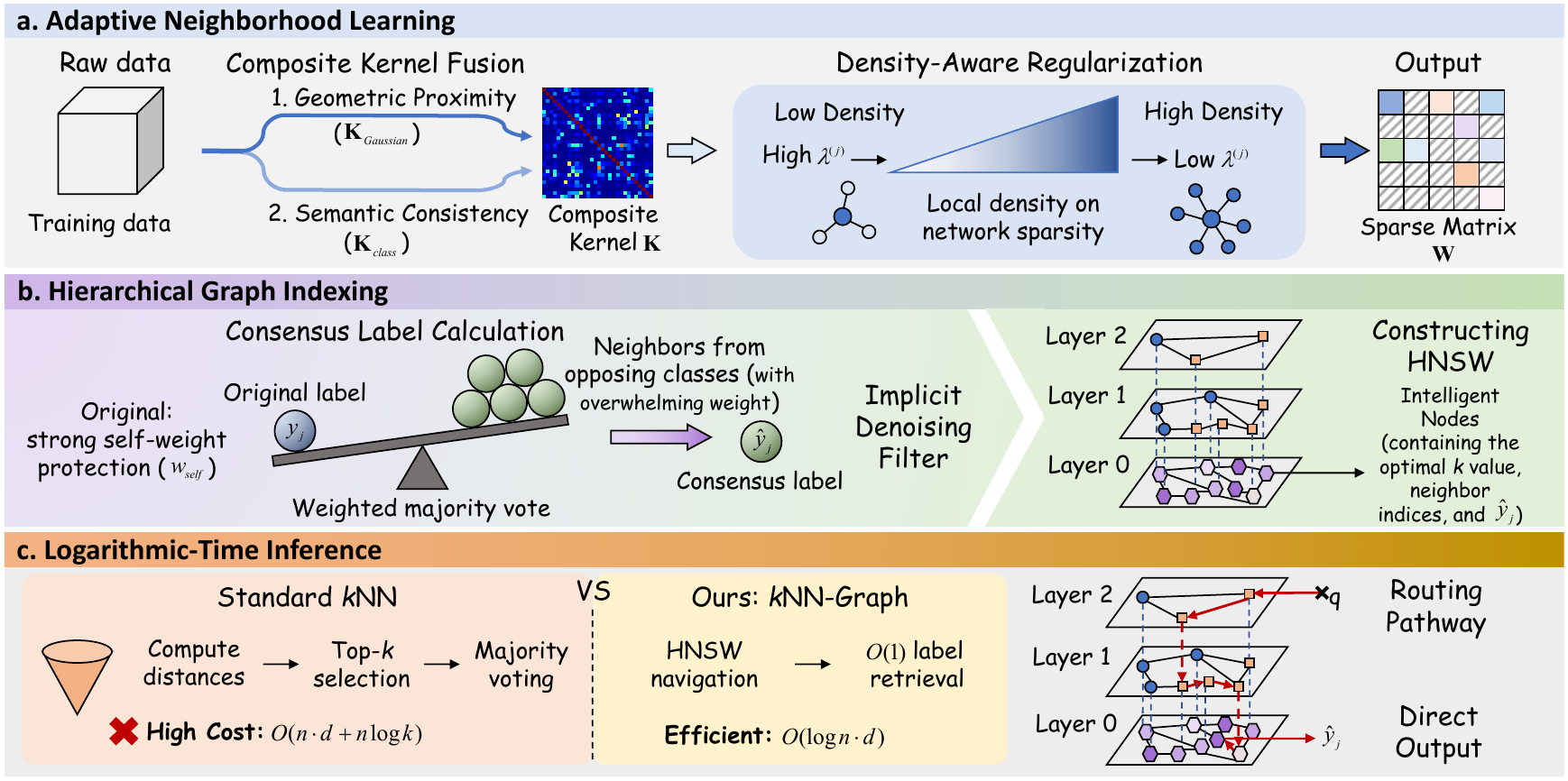}
	\caption{\textbf{Architecture of the $k$NN-Graph framework.} The framework decouples neighborhood learning from the inference phase across three stages. \textbf{a} Adaptive Neighborhood Learning. Training data are mapped into a composite kernel ($\mathbf{K}$) space capturing geometric proximity ($\mathbf{K}_{\text{Gaussian}}$) and semantic consistency ($\mathbf{K}_{\text{class}}$). A density-aware regularization mechanism determines the sparsity parameter ($\lambda^{(j)}$) based on local density, yielding the sparse representation matrix $\mathbf{W}$. \textbf{b} Hierarchical Graph Indexing. A Hierarchical Navigable Small World (HNSW) structure is constructed using precomputed nodes. For each node, a weighted consensus label ($\hat{y}_j$) is calculated. This process functions as an implicit denoising filter, utilizing a self-weight parameter ($w_{\text{self}}$) to balance against neighboring weights and mitigate the noise of the original label ($y_j$). \textbf{c} Logarithmic-Time Inference. In contrast to standard $k$-nearest neighbors ($k$NN) methods that require exhaustive runtime distance computations, top-$k$ sorting, and majority voting (left), the $k$NN-Graph (right) employs greedy routing (indicated by solid and dashed red arrows) through the HNSW hierarchy for a query point ($\mathbf{x}_q$) to locate the nearest node. The precomputed consensus label ($\hat{y}_j$) is directly retrieved, achieving an inference complexity of $O(\log n \cdot d)$, where $n$ is the number of training samples and $d$ is the feature dimension.}
	\label{fig5}

\end{figure*}



\end{document}